\documentclass[sigconf]{aamas} 
\usepackage{bbm}
\usepackage{algpseudocode}
\usepackage{algorithm}
\usepackage{multirow}
\usepackage{caption}
\usepackage{subcaption}
\newtheorem{theorem}{Theorem}[section]

\newtheorem{proposition}[theorem]{Proposition}

\usepackage{balance} % for balancing columns on the final page
\usepackage{xcolor}
\setcopyright{none}
\acmConference[ALA '23]{Proc.\@ of the Adaptive and Learning Agents Workshop (ALA 2023)} {May 29-30, 2023}{London, UK, \url{https://alaworkshop2023.github.io/}}{Cruz, Hayes, Wang, Yates (eds.)} \copyrightyear{2023} \acmYear{2023} \acmDOI{} \acmPrice{} \acmISBN{} 

\acmSubmissionID{60}

\title[Approximate Shielding of Atari Agents for Safe Exploration]{Approximate  Shielding of Atari Agents for Safe Exploration}

\author{Alexander W. Goodall}
\affiliation{
  \institution{Imperial College London}
  \city{London}
  \country{United Kingdom}}
\email{a.goodall22@imperial.ac.uk}

\author{Francesco Belardinelli}
\affiliation{
  \institution{Imperial College London}
  \city{London}
  \country{United Kingdom}}
\email{francesco.belardinelli@imperial.ac.uk}

\begin{abstract}
Balancing exploration and conservatism in the constrained setting is an important problem if we are to use reinforcement learning for meaningful tasks in the real world. In this paper, we propose a principled algorithm for \emph{safe exploration} based on the concept of \emph{shielding}. Previous approaches to shielding assume access to a safety-relevant abstraction of the environment or a high-fidelity simulator. Instead, our work is based on \emph{latent shielding} - another approach that leverages world models to verify policy roll-outs in the latent space of a learned dynamics model. Our novel algorithm builds on this previous work, using safety critics and other additional features to improve the stability and farsightedness of the algorithm. We demonstrate the effectiveness of our approach by running experiments on a small set of Atari games with state dependent safety labels. We present preliminary results that show our approximate shielding algorithm effectively reduces the rate of safety violations, and in some cases improves the speed of convergence and quality of the final agent.
\end{abstract}


\newcommand{\BibTeX}{\rm B\kern-.05em{\sc i\kern-.025em b}\kern-.08em\TeX}

\begin{document}

%%% The following commands remove the headers in your paper. For final 
%%% papers, these will be inserted during the pagination process.

\pagestyle{fancy}
\fancyhead{}

%%% The next command prints the information defined in the preamble.

\maketitle 

%%%%%%%%%%%%%%%%%%%%%%%%%%%%%%%%%%%%%%%%%%%%%%%%%%%%%%%%%%%%%%%%%%%%%%%%

\section{Introduction}
Reinforcement learning (RL) \cite{sutton2018reinforcement} has become a principled and powerful tool for training agents to complete tasks in complex and dynamic environments. While RL promises a lot in theory, it unfortunately comes with no guarantees on worst-case performance. In safety-critical applications such as healthcare, robotics, autonomous driving and industrial control systems, it is imperative that decision making algorithms avoid unsafe or harmful situations \cite{amodei2016concrete}. Formal verification \cite{baier2008principles} poses as a mathematically precise technique for verifying system performance and can be used to verify that learned policies respect safety-constraints during training and deployment. 

Recently there has been increasing interest in applying model-based RL (MBRL) algorithms in the constrained setting. This increase in interest can be attributed in part to exciting developments in MBRL \cite{hafner2020mastering, hafner2023mastering} and the superior sample complexity of model-based approaches \cite{hafner2019dream, janner2019trust}. With better sample-complexity, MBRL algorithms should in theory commit far fewer safety violations during training than their model-free counterparts. This is important in the problem of \emph{safe exploration} \cite{amodei2016concrete} where collecting experience is costly and unsafe behaviour can lead to catastrophic consequences in the real world. 

In this work we focus on a method for safe exploration called \emph{shielding} \cite{alshiekh2018safe, jansen2018shielded}.
In its original form, shielding forces hard constraints on the actions performed by the agent to ensure that the agent stays within a verified boundary on the state space. To compute this boundary we typically require a safety-relevant abstraction of the environment that is compact enough to efficiently perform exact verification techniques. Instead we opt to be less restrictive and make minimal assumptions about what we have access to a priori. As in previous work \cite{he2021androids}, we only assume that there exists some expert labelling of the states and we do not have access to a compact model or a safety-relevant abstraction of the environment. The key motivation for making these minimal assumptions is to obtain a more general algorithm that can be applied in many real-life applications where an abstraction is typically not available, as the system might be too complex or unknown in advance.

\begin{figure}[h!]
	\centering
	\includegraphics[width=.25\textwidth]{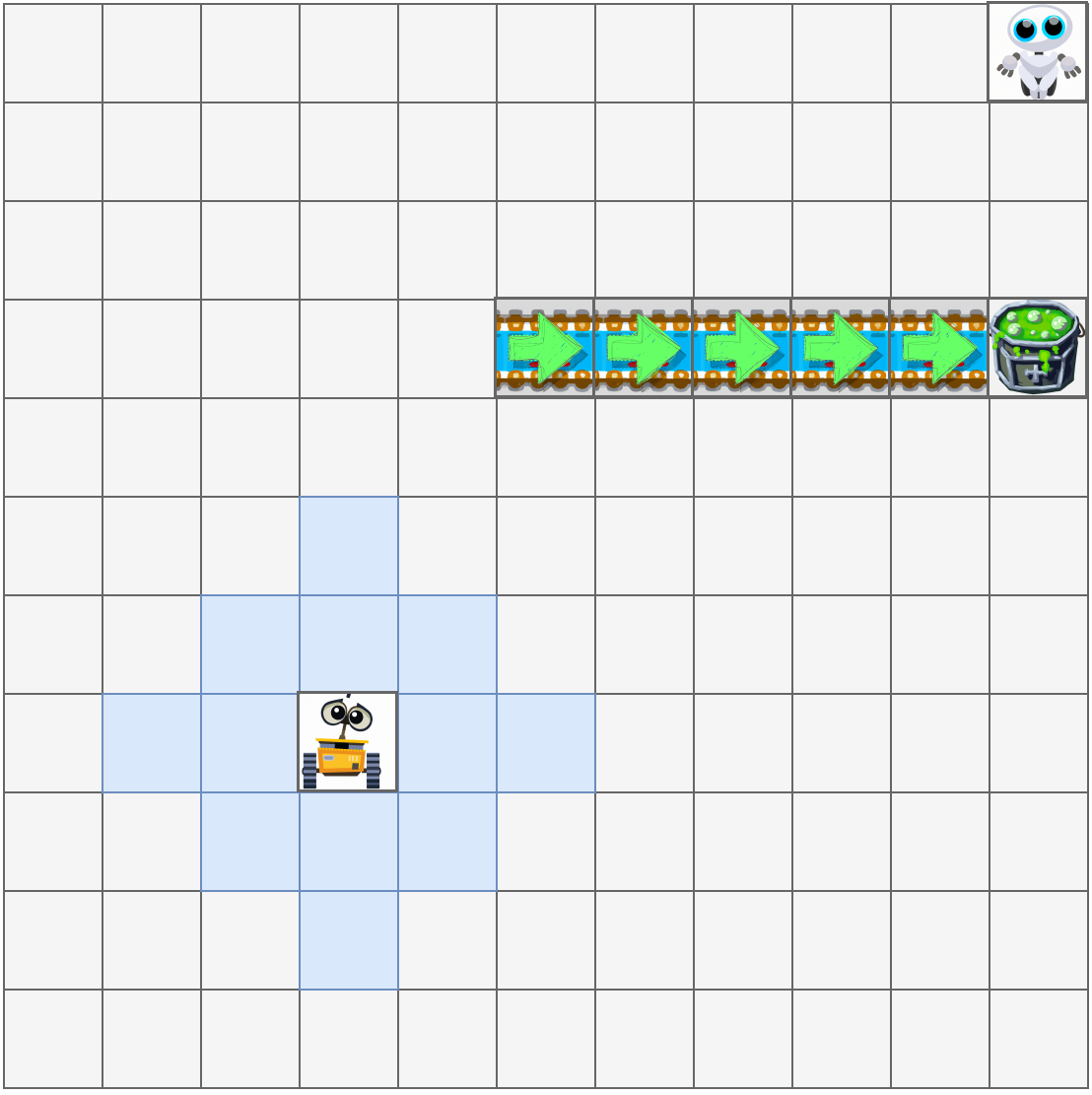}
	\caption{Simple grid-world with goal (Eve) in the top right corner. With a Manhattan distance look-ahead of 2 (blue squares) Wall-E is forever doomed to fall into the acid during exploration, as he is unable to determine that the conveyor belt leads to an unavoidable unsafe state. With safety critics Wall-E can learn the cost value of the conveyor belt squares and avoid them without a further look-ahead horizon. \textsuperscript{*}}
        \raggedright
        \small\textsuperscript{*} This image was created with the assistance of DALL·E 2
	\label{fig:fig1}
    \Description{A simple grid world environment were Wall-E the agent is trying to get to Eva the goal. Wall-E only has a Manhattan look-ahead horizon of 2. There is a conveyor belt with length 5 that ends in acid, without a large look-ahead distance wall-E is doomed to fall into the acid if he steps on to the conveyor belt during exploration. }
\end{figure}

Bounded prescience shielding (BPS) \cite{giacobbe2021shielding} is an approach to shielding that removes the requirement of access to a compact representation or abstraction of the environment. Instead, BPS assumes access to a black-box simulator of the environment which can be queried for look-ahead shielding of the learned policy. \citeauthor{giacobbe2021shielding} demonstrated that pre-trained state-of-the-art model-free Atari agents consistently violate safety-constraints provided by domain experts. And that BPS with a look-ahead horizon of $H=5$ reduced the rate of several shallow safety properties. In this paper we use the same state-dependent safety-labels for Atari games provided by \citeauthor{giacobbe2021shielding}, although we note that our approach has several distinct advantages over BPS: (1) we do not assume access to a black-box model to assist decision making, (2) we are able to apply our shielding algorithm during training without substantial computational overhead, (3) we are able to look-ahead further into the future ($>5$) with deeper model roll-outs and safety critics. 

In a similar fashion to \emph{latent shielding} \cite{he2021androids}, we use world models \cite{ha2018world, hafner2019dream, hafner2020mastering} to learn a dynamics model of the environment for policy optimisation and approximate shielding. The key differences between \emph{latent shielding} and our approach are outlined in Section \ref{sec:approxmodelcheck}. Most notably we utilise safety critics, which are crucial for obtaining further look-ahead ability without explicitly increasing the shielding horizon, see Fig. \ref{fig:fig1}.

\paragraph{Contributions}
Our main contributions are summarised as follows: (1) we augment \emph{latent shielding} \cite{he2021androids} with safety critics used to bootstrap the end of imagined trajectories for look-ahead shielding further into the future, (2) we use twin delayed target critics to reduce the overestimation of expected costs and reduce overly conservative behaviour, (3) we ground our approach in a logical formalism, namely probabilistic computation tree logic (PCTL) \cite{baier2008principles}, (4) we derive PAC-style bounds on the probability of accurately estimating a constraint violation under the assumption of perfect transition dynamics, (5) we empirically show that our approach dramatically reduces the rate of safety violations on a small set of Atari games with state-dependent safety labels and in some cases our algorithm greatly improves the speed of convergence and quality of the learned policy with respect to accumulated reward.

%%%%%%%%%%%%%%%%%%%%%%%%%%%%%%%%%%%%%%%%%%%%%%%%%%%%%%%%%%%%%%%%%%%%%%%%

\section{Preliminaries}

In this section we describe the relevant background material and notation required to understand the main results of this paper. We start by introducing the problem setup and the specification language used to formalise the notion of safety used throughout this paper. We then continue with an outline of the world model components and the policy optimisation scheme.

\subsection{Problem Setup}

Atari games in the Arcade Learning Environment (ALE) \cite{machado18arcade} are built on top of the Atari 2600 Stella emulator. The emulator manipulates 128-bytes of RAM which represent the underlying state of the game. However, Agents typically only observe $3 \times 210 \times 160$ dimensional tensors representing each of the pixel values of the screen. Therefore, we model the system as a partially observed Markov decision processes (POMDP) \cite{puterman1990markov}, which in this case is more appropriate than the traditional MDP formulation.

For our purposes we also extend the POMDP tuple with an additional labelling function \cite{baier2008principles}. Formally, we define a POMDP as a tuple $\mathcal{M} = (S, A, p, \iota_{init}, R, \Omega, O, AP, L)$ where, $S$ is a finite set of states, $A$ is a finite set of actions, $p : S \times A \times S \to [0, 1]$ is the probabilistic state-action transition function, $\iota_{init} : S \to [0, 1]$ is the initial state distribution such that $\sum_{s \in S} \iota_{init}(s) = 1$, $R : S \times A \to \mathbb{R}$ is the reward function, $\Omega$ is a finite set of observations, $O : S \times A \times \Omega \to [0, 1]$ is the observation probabilistic function, which defines the probability of an observation conditional on the previous state-action pair, $\textit{AP}$ is a set of atomic propositions which maps to the set of states by an `expert' labelling function $L : S \to 2^{\textit{AP}}$.

In particular, at each timestep $t$ the agent receives an observation $o_t \in \Omega$, a reward $r_t$ and a set of labels $L(s_t) \in 2^{\textit{AP}}$. Given some state formula $\Phi$, the agent can determine if the underlying state $s_t$ satisfies $\Phi$ with the following relation,
\begin{equation*}
    \begin{array}{@{}r@{{}\mathrel{}}c@{\mathrel{}{}}l@{}}
        s \models \textrm{true} & \text{for all} & s \in S\\
    s \models a & \text{iff} & a \in L(s)\\
    s \models \neg \Phi & \text{iff} & s \not \models \Phi\\
    s \models \Phi_1 \land \Phi_2 & \text{iff} & s \models \Phi_1 \land s \models \Phi_2
    \end{array}
\end{equation*}
% make this definition for optimal policy consistent
The goal is to find a policy $\pi$ that maximises expected reward, i.e. $\pi^* = \arg\max_{\pi} \mathbb{E}[\sum_{t=1}^{\infty} \gamma^{t-1} R(s_t, \pi(s_t))]$, while minimising violations of the state formula $\Phi$ (that encodes the safety-constraints) during training. Here $\gamma$ is the discount factor \cite{sutton2018reinforcement}.
% motivation
%\ag{I've changed this just to formalise POMDPs to avoid confusion.}
%\ag{I've introduced the notion of state formula here for clarity}

%\fb{shall we introduce POMDPs directly? or do we need MDPs as well?}
%\ag{We've implicity introduced POMDPs by saying we extend the tuple with an observation space and observation funciton, we could be more explicit if you think it is necessary}
%\fb{It's better to have all definitions as explicit as possible. In this case the reader might be left in doubt as to whether the system is modelled as an MDP or a POMDP.} 

\subsection{Probabilistic Computation Tree Logic}

Probabilistic computation tree logic (PCTL) is a branching time temporal logic that extends CTL with probabilistic quantifiers \cite{baier2008principles}. PCTL is particularly useful for specifying reachability and safety properties for discrete stochastic systems which makes it useful for our purposes. A valid PCTL formula can be constructed as follows,
\begin{align*}
    \Phi ::= & \textrm{true} \mid a \mid \neg \Phi \mid \Phi \land \Phi \mid \mathbb{P}_{J}(\phi) \\
    \phi ::= & X \Phi \mid \Phi U \Phi \mid \Phi U^{\leq n} \Phi
\end{align*}
where $a \in AP$ is an atomic proposition, negation ($\neg$) and conjunction ($\land$) are the familiar logical operators, $J \subset [0, 1]$, $J \neq \varnothing$ is a non-empty subset of the unit interval, and next ($X$), until ($U$) and bounded until ($U^{\leq n}$) are temporal operators. We distinguish here between state formula $\Phi$ and path formula $\phi$ which are interpreted over states and paths respectively. 

We write $s \models \Phi$ to indicate that a state $s$ satisfies a state formula $\Phi$, where the satisfaction relation is defined as before, see \cite{baier2008principles} for details. Similarly, we can define the satisfaction relation for path formula $\phi$, this is given in the next section for the specific fragment of PCTL that we require. Also note that the common operators eventually ($\lozenge$) and always ($\square$) and their bounded counter parts ($\lozenge^{\leq n}$ and  $\square^{\leq n}$) can be defined in a familiar way, see \cite{baier2008principles}.
%\fb{probably we can remove the following:
%\begin{align}
%    \lozenge \:\phi = & \textrm{true} \cup \phi  \\
%\square \:\phi = &  \neg \lozenge \neg \phi \\
%\lozenge^{\leq n} \:\phi = & \textrm{true} \cup^{\leq n} \phi  \\
%\square^{\leq n} \:\phi = &  \neg \lozenge^{\leq n} \neg \phi
%\end{align}
%}

The reason behind using PCTL as our safety specification language is because it allows us to meaningfully trade-off safety and progress by specifying the probability with which we force the agent to satisfy to a given temporal logic formula.

%\subsection{Shielding}

%\fb{this section could go into the related work section or the intro, as we don't introduce any technical notion.
%By doing so, we would link together the sections on PCTL and bounded safety.}
%\ag{Yes I agree, I've moved the first paragraph to the related work section and the second paragraph on BPS to the introduction - I think this makes the most sense}

\subsection{Bounded Safety}

The notion of bounded safety for Atari agents introduced by \citeauthor{giacobbe2021shielding} can be straightforwardly grounded in PCTL. Consider some fixed (stochastic) policy $\pi : O \times A \to [0, 1]$ and POMDP $\mathcal{M} = (S, A, p, \iota_{init}, R, \Omega, O, AP, L)$. Together $\pi$ and $\mathcal{M}$ define a transition system $\mathcal{T} : S \times S \to [0, 1]$, where $\sum_{s' \in S} \mathcal{T}(s, s') = 1$. A finite trace with length $n$ of the transition system $\mathcal{T}$, is a sequence of states $s_0 \to s_1 \to ... \to s_n$ denoted $\tau$, the $i^\textrm{th}$ state of $\tau$ is given by $\tau[i]$. A trace $\tau$ satisfies bounded safety if and only if all of its states satisfy the state formula $\Phi$ that encodes the safety constraints. Formally,
\begin{eqnarray}
    \tau \models \square^{\leq n} \Phi & \text{iff} & \text{for all } 0 \leq i \leq n, \tau[i] \models \Phi
\end{eqnarray}
for some bounded look-ahead $n$. Now in PCTL we can say that a state $s \in S$ satisfies $\varepsilon$-bounded safety as follows,
% might need to double check this definition
\begin{multline}
    s \models \mathbb{P}_{1 - \varepsilon}(\square^{\leq n} \Phi ) \;\; \text{iff}\\ \mu_s(\{\tau \mid \tau[0] = s, \text{ for all }0 \leq i \leq n, \tau[i] \models \Phi\}) \in [1- \varepsilon, 1] \label{eq:eboundedsafety}
\end{multline}
where $\mu_s$ is a well-defined probability measure induced by the transition system $\mathcal{T}$, over the set of traces staring from $s$ and with finite length $n$, see \cite{baier2008principles} for details. We denote $\mu_{s\models \phi}$ as shorthand for the measure $ \mu_s(\{\tau \mid \tau[0] = s, \text{ for all }0 \leq i \leq n, \tau[i] \models \Phi\})$, where $\phi ::= \square^{\leq n} \Phi $ is the path formula we care about.  By framing bounded safety in this way, we obtain a meaningful way to trade off safety and progress with the $\varepsilon$ parameter. 

\subsection{World Models}

To learn a world model for behaviour learning and look-ahead shielding we leverage DreamerV2 \cite{hafner2020mastering}, which was used to master Atari games in the ALE \cite{machado18arcade}. DreamerV2 is composed of the following components: an image encoder $z_t \sim q_{\theta}(z_t \: | \: o_t, h_t)$ that learns a posterior latent representation conditional on the current observation $o_t$ and recurrent state $h_t$, the recurrent state space model (RSSM) \cite{hafner2019learning} which is a mixture of deterministic and stochastic categorical latents, and the image, reward and discount predictors.

The RSSM consists of two main components: the recurrent model $h_t = f_{\theta}(h_{t-1}, z_{t-1},a_{t-1})$, which computes the next deterministic latents given the past state $ s_{t-1} = (h_{t-1}, z_{t-1})$ and action $a_{t-1}$, and the transition predictor $\hat z_t \sim p_{\theta}(\hat z_t \: | \: h_t)$, which is used as the prior distribution over 
 the stochastic latents conditional on the deterministic latents.

The image predictor or decoder $\hat o_t \sim p_{\theta}(\hat o_t \: | \: h_t, z_t)$ is trained to predict the current observation $o_t$ with a reconstruction loss. The image predictor provides useful self-supervised gradients that help the world model learn a structured latent space for effective policy optimisation \cite{hafner2020mastering}. The reward predictor $\hat r_t \sim p_{\theta}(\hat r_t \: | \: h_t,z_t )$ and discount predictor $\hat \gamma_t \sim p_{\theta}(\hat \gamma_t \: | \:h_t, z_t)$, also provide useful self-supervised gradients. However, they are primarily used to construct targets for policy optimisation.

All components of the world model are implemented as neural networks and jointly trained with backpropagation and straight through gradients \cite{bengio2013estimating}. In addition, KL-balancing \cite{hafner2020mastering} is used to stop the prior and posterior being regularised at the same rate to prevent instability during training.

\subsection{Behaviour Learning}

In DreamerV2 \cite{hafner2020mastering}, policy optimisation is performed entirely on experience `imagined' by rolling out the world model with a fixed (stochastic) policy. A replay buffer $\mathcal{D}$ is used to retain experience from the real environment. At each training step a batch $B$ is sampled from the replay buffer $\mathcal{D}$ and the RSSM is used to sample sequences of compact latent states $\hat s_{1:H}$, using each of the observations in $B$ as a starting point. Here $H$ refers to the `imagination' horizon, which is typically set to a relatively small number ($H=15$) to avoid compounding model errors that are likely to harm the learned policy.

The task policy $\pi^{\textrm{task}}$ parameterised by $\psi^{\textrm{task}}$ is trained to maximise accumulated reward. In addition, a task critic $v^{\textrm{task}}$ parameterised by $\xi^{\textrm{task}}$ is used to guide the learning of the policy. TD-$\lambda$ targets \cite{sutton2018reinforcement} are constructed by rolling out the world model with the task policy $\pi^{\textrm{task}}$,
\begin{equation}
    V_t^{\text{task}, \lambda} = \hat r_t + \hat \gamma_t \begin{cases}
    (1 - \lambda)v^{\text{task}}(\hat s_{t+1}) + \lambda V_{t+1}^{\text{task},\lambda} & \text{if $t < H$,}\\
    v^{\text{task}}(\hat s_{H}) &\text{if $t = H$} \label{eq:tasktargets}
\end{cases}
\end{equation}
The $\lambda$ parameter trades of the bias and variance of the estimate, with $\lambda= 0.0$ giving the high variance n-step Monte-Carlo return and $\lambda =1.0$ giving the high bias one-step return. The task critic $v^{\textrm{task}}$ is regressed towards the value estimates with the following loss function,
\begin{equation}
    \mathcal{L}(\xi^{\text{task}}) = \mathbb{E}_{\pi^{\text{task}}, p_{\theta}} \left[\sum^{H-1}_{t=1} \frac{1}{2} (v^{\text{task}} (\hat s_t) - sg(V_t^{\text{task},\lambda}))^2\right] \label{eq:taskcriticloss}
\end{equation}
where the $sg(\cdot)$ operator stops the flow of gradients to the input argument. The task policy $\pi^{\text{task}}$ is trained with reinforce gradients \cite{sutton2018reinforcement} and an entropy regulariser to encourage exploration. In addition, the difference of the TD-$\lambda$ targets $V^{\text{task},\lambda}$ and the critic estimates $v^{\text{task}}$ are used as a baseline to reduce the variance of the reinforce gradients. This gives the following loss function for the task policy $\pi^{\text{task}}$,
\begin{multline}
    \mathcal{L}(\psi^{\text{task}}) = \\ \mathbb{E}_{\pi^{\text{task}}, p_{\theta}}\bigg[\sum^{H-1}_{t=1} \underbrace{- \log \pi^{\text{task}}(a_t \: | \: \hat s_t)sg(V_t^{\text{task},\lambda} - v^{\text{task}}(\hat s_t))}_{\text{reinforce}} \\ \underbrace{- \eta H(\pi^{\text{task}}(\cdot \: | \: \hat s_t))}_{\text{entropy}} \bigg] \label{eq:taskpolicyloss}
\end{multline}

%%%%%%%%%%%%%%%%%%%%%%%%%%%%%%%%%%%%%%%%%%%%%%%%%%%%%%%%%%%%%%%%%%%%%%%%

\section{Approximate Shielding}
\label{sec:approxmodelcheck}
In this section we introduce our approximate shielding algorithm for Atari agents. The general idea is to learn a world model for task policy optimisation, safe policy synthesis and bounded look-ahead shielding. The world model of choice is DreamerV2 \cite{hafner2020mastering} which has demonstrated state-of-the-art performance on the Atari benchmark. While our approach is similar to \emph{latent shielding} \cite{he2021androids} we note that it has following key differences:
\begin{itemize}
    \item We learn a cost predictor to estimate state dependent costs, rather than a labelling function $L_{\vartheta} : S \to \{\textrm{safe}, \textrm{unsafe} \}$. 
    \item We train a safe policy to minimise expected costs, which is used as the backup policy if a safety-violation is detected.
    \item We use safety critics to obtain further look-ahead capabilities without having to roll-out the world model further into the future.
    \item We don't need to use intrinsic punishment \cite{alshiekh2018safe} or any sort of shield introduction schedule \cite{he2021androids}.
    \item We test our approach in a much more sophisticated domain, specifically the ALE \cite{machado18arcade}. 
\end{itemize}
In what follows, we describe the notable components used in our approach, followed by a precise description of the shielding procedure and an outline of the full learning algorithm. However, we will first present some PAC-style bounds on the probability of accurately predicting a constraint violation using our the shielding procedure. It should then become clear in what sense our algorithm approximate. Specifically, `approximate' comes from the fact that we use a learned approximation of the true environment dynamics and we use Monte-Carlo estimation to predict constraint violations.

\subsection{Probabilistic Guarantees}
Recall that to ensure $\varepsilon$-bounded safety we are interested in verifying PCTL formula of the form $ \mathbb{P}_{1 - \varepsilon}(\square^{\leq n} \Phi )$, where $\Phi$ is the state formula that encodes the safety-constraints. To do so we fix the task policy $\pi^{\textrm{task}}$ and the learned world model $p_{\theta}$ to obtain an approximate transition system $\widehat{\mathcal{T}} : S \times S \to [0, 1]$. Even with the `true' transition system $\mathcal{T} : S \times S \to [0, 1]$, exact PCTL verification is $O(\textrm{poly}(size(\mathcal{T}))\cdot n \cdot |\Phi|)$, which is much too big for Atari games. Instead we rely on Monte-Carlo estimation of the measure $\mu_{s\models \phi}$ by sampling traces $\tau$ from the approximate transition system $\widehat{\mathcal{T}}$.

\begin{proposition}
Given access to the `true' transition system $\mathcal{T}$, with probability $1 - \delta$ we can estimate the measure $\mu_{s\models\phi}$ up to some approximation error $\epsilon$, by sampling $m$ traces $\tau \sim \mathcal{T}$, provided,

\begin{equation}
    m \geq \frac{1}{2\epsilon^2} \log\left(\frac{2}{\delta}\right) \label{eq:boundonm}
\end{equation}
\label{prop:boundonm}
\end{proposition}
\begin{proof}
The proof is a straightforward application of Hoeffding's inequality.
We can estimate $\mu_{s\models\phi}$ by sampling $m$ traces $\langle\tau_j\rangle^m_{j=1}$ from $\mathcal{T}$. Let $X_1, ..., X_m$ be indicator r.v.s such that,

\begin{equation}
    X_j = \begin{cases}
        1 & \text{if $\tau_j \models \square^{\leq n} \Phi$,}\\
        0 & \text{otherwise}
    \end{cases}
\end{equation}
Let,
\begin{equation}
    \hat \mu_{s\models\phi} = \frac{1}{m} \sum^m_{j=1}X_j, \; \text{where} \;\: \mathbb{E}_{\mathcal{T}}[\hat \mu_{s\models\phi}] = \mu_{s\models\phi} 
\end{equation}
Then by Hoeffding's inequality,
$$\mathbb{P}\left[|\hat \mu_{s\models\phi} - \mu_{s\models\phi} | \geq \epsilon \right]\leq 2\exp\left( - 2 m \epsilon^2 \right)$$
Bounding the RHS from above with $\delta$ and rearranging completes the proof.
\end{proof}

With these probabilistic guarantees we can set $m$ appropriately for some domain specific requirements. In the following proposition, we demonstrate how we may be sure that a given state $s$ satisfies $\varepsilon$-bounded safety given our estimate $\hat \mu_{s\models\phi}$.
\begin{proposition}
Suppose we have an estimate $\hat \mu_{s\models\phi} \in [\mu_{s\models\phi} - \epsilon, \mu_{s\models\phi} + \epsilon]$, if $\hat \mu_{s\models\phi} \in [1 - \varepsilon + \epsilon, 1]$ then it must be the case that $\mu_{s\models\phi} \in [1 - \varepsilon, 1]$ and that $s \models \mathbb{P}_{1 - \varepsilon}(\square^{\leq n} \Phi )$.
\end{proposition}

Note that $\varepsilon$ is the bounded safety parameter used to trade-off exploration and progress and $\epsilon$ is the approximation error from Proposition \ref{prop:boundonm}.
%\begin{proof}
%    The proof is a trivial by Eq. \ref{eq:eboundedsafety}.
%\end{proof}
\begin{proof}
Suppose $\hat \mu_{s\models\phi} \in [\mu_{s\models\phi} - \epsilon, \mu_{s\models\phi} + \epsilon]$, $\hat \mu_{s\models\phi} \in [1 - \varepsilon + \epsilon, 1]$ and $\mu_{s\models\phi} \not\in [1 - \varepsilon, 1]$. Then $\hat \mu_{s\models\phi} - \mu_{s\models\phi} > \epsilon $ which contradicts $\hat \mu_{s\models\phi} \in [\mu_{s\models\phi} - \epsilon, \mu_{s\models\phi} + \epsilon]$. This implies that indeed $\mu_{s\models\phi} \in [1 - \varepsilon, 1]$ and that $s \models \mathbb{P}_{1 - \varepsilon}(\square^{\leq n} \Phi )$ by Eq. \ref{eq:eboundedsafety}.
\end{proof}
It is important to note that checking the condition $\hat \mu_{s\models\phi} \in [1 - \varepsilon + \epsilon, 1]$ could lead to overly conservative behaviour, if $\epsilon$ is not very small. This is because for $\mu_{s\models\phi} \in [1- \varepsilon, 1-\varepsilon + \epsilon]$ we may falsely predict that $s \not \models \mathbb{P}_{1 - \varepsilon}(\square^{\leq n} \Phi )$ with some probability up to $1 - \delta$. Instead we could check that $\hat \mu_{s\models\phi} \in [1 - \varepsilon - \epsilon, 1]$, although this may lead to overly permissive behaviour. In words, the former configuration admits no false positives and the latter admits no false negatives (with probability $1- \delta$). Either configuration can be used, although we opt for the former.

To get similar bounds for the approximate transition system $\widehat{\mathcal{T}}$ we can try to get a bound on the total variation (TV) distance between $\widehat{\mathcal{T}}$ and $\mathcal{T}$. However, this is left for future work.
 
\subsection{RSSM with Costs}

We augment the RSSM of DreamerV2 \cite{hafner2020mastering} with a cost predictor $\hat c_t \sim p_{\theta}(\hat c_t \: | \: h_t, z_t)$ used to predict state dependent costs and a safety-discount predictor $\hat \gamma^{\text{safe}}_t \sim p_{\theta}(\hat \gamma^{\text{safe}}_t \: | \: h_t, z_t)$ which is used to help improve the stability of the safety critics. 

In the same fashion as the reward predictor $\hat r_t \sim p_{\theta}(\hat r_t \: | \: h_t,z_t )$ the cost predictor $\hat c_t \sim p_{\theta}(\hat c_t \: | \: h_t, z_t)$ parameterises a Gaussian distribution. We construct targets for the cost predictor as follows,
\begin{equation}
    c_t = \begin{cases}
    0, & \text{if $s_t \models \Phi$} \\
    C, & \text{otherwise}
\end{cases} \label{eq:costtarget}
\end{equation}
where $s_t$ refers to the true underlying state of the environment, $\Phi$ is the state formula that encodes the safety-constraints, and $C>0$ is an arbitrary hyperparameter that determines the cost incurred at a violating state. Using a cost predictor in this way allows the agent to distribute its uncertainty about a constraint violation over several consecutive states.

The safety-discount predictor $\hat \gamma^{\text{safe}}_t \sim p_{\theta}(\hat \gamma^{\text{safe}}_t \: | \: h_t, z_t)$ is a binary classifier trained, in a similar way, to predict if a state is violating or not. We construct targets for the safety-discount predictor as follows,
\begin{equation}
    \gamma^{\text{safe}}_t = \begin{cases}
			\gamma, & \text{if $s_t \models \Phi$}\\
            0, & \text{otherwise}
		 \end{cases} \label{eq:safetydiscounttarget}
\end{equation}
The key purpose of the safety-discount factor is to reduce the overestimation of the expected costs by the safety critics. Using the safety-discount predictions $\hat \gamma^{\textrm{safe}}_t$ to construct targets for the safety critics, instead of the usual discount predictions $\hat \gamma_t$, effectively transforms the MDP into one where violating states are terminal states. This means the safety critics should always be upper bounded by $C$. The full RSSM loss function can now be written as follows,
\begin{multline}
    \mathcal{L}(\theta) = \mathcal{L}_{\text{image}} + \mathcal{L}_{\text{reward}} + \mathcal{L}_{\text{discount}} + \mathcal{L}_{\text{cost}} \\ + \mathcal{L}_{\text{safe-discount}} + \mathcal{L}_{\text{KL-B}} \label{eq:rssmloss}
\end{multline}

\subsection{Safe Policy}

The safe policy $\pi^{\text{safe}}$ is used as the backup policy if we detect that the task policy $\pi^{\text{task}}$ is likely to commit a safety violation in the next $T$ steps. Since we have no access to an abstraction of the environment we cannot synthesise a shield before training and so the safe policy must be learned.

The safe policy $\pi^{\text{safe}}$ is only concerned with minimising expected costs and so we use the cost predictor $\hat c_t \sim p_{\theta}(\hat c_t \: | \: h_t, z_t)$ to construct TD-$\lambda$ targets as follows,
\begin{equation}
    V_t^{\text{safe}, \lambda} = \hat c_t + \hat \gamma_t \begin{cases}
    (1 - \lambda)v^{\text{safe}}(\hat s_{t+1}) + \lambda V_{t+1}^{\text{safe},\lambda} & \text{if $t < H$,}\\
    v^{\text{safe}}(\hat s_{H}) &\text{if $t = H$} \label{eq:safecritictargets}
\end{cases}
\end{equation}
The safe critic $v^{\text{safe}}$ parameterised by $\xi^{\text{safe}}$ is regressed towards the TD-$\lambda$ targets with a similar loss function as before,
\begin{equation}
    \mathcal{L}(\xi^{\text{safe}}) = \mathbb{E}_{\pi^{\text{safe}}, p_{\theta}} \left[\sum^{H-1}_{t=1} \frac{1}{2} (v^{\text{safe}}(\hat s_t) - sg(V_t^{\text{safe},\lambda}))^2\right] \label{eq:safecriticloss}
\end{equation}
The safe policy $\pi^{\text{safe}}$ parameterised by $\psi^{\text{safe}}$ is also trained with biased reinforce gradients and an entropy regulariser as before,
\begin{multline}
    \mathcal{L}(\psi^{\text{safe}}) = \mathbb{E}_{\pi^{\text{safe}}, p_{\theta}}\bigg[\sum^{H-1}_{t=1} \underbrace{ \log \pi^{\text{safe}}(a_t \: | \: \hat s_t)sg(V_t^{\text{safe},\lambda} - v^{\text{safe}}(\hat s_t))}_{\text{reinforce}} \\ \underbrace{- \eta H(\pi^{\text{safe}}(\cdot \: | \: \hat s_t))}_{\text{entropy}} \bigg] \label{eq:safepolicyloss}
\end{multline}
Note that the sign is flipped here, so that the safe policy $\pi^{\text{safe}}$ minimises expected costs rather than maximises them.
\subsection{Safety Critics}

Safety critics estimate the expected costs under the task policy $\pi^{\text{task}}$. They give us an idea of how safe specific states are under the task policy state distribution. Additionally, we can use them to bootstrap the end of `imagined' trajectories for further look-ahead capabilities.  

To estimate the expected costs under the task policy $\pi^{\text{task}}$ we use two safety critics $v^C_1$ and $v^C_2$ parameterised by $\xi^C_1$ and $\xi^C_2$ respectively. To prevent overestimation, the safety critics are jointly trained with a TD3-style algorithm \cite{fujimoto2018addressing} to estimate the following quantity,
\begin{equation}
    \mathbb{E}_{\pi^{\text{task}}, p_{\theta}}\left[ \sum^{\infty}_{t=1}  (\hat{\gamma}^{\text{safe}}_t)^{t-1} \cdot \hat c_t \right]
\end{equation}
Each of the two safety critics $v^C_1$ and $v^C_2$, has its own target critic ${v^C_1}'$ and ${v^C_2}'$, that are updated periodically with slow updates. The TD-$\lambda$ targets are constructed by taking a minimum of the two target critics ${v^C_1}'$ and ${v^C_2}'$ as follows,
\begin{equation}
    V_t^{C, \lambda} = \hat c_t + \hat \gamma^{\text{safe}}_t \begin{cases}
    (1 - \lambda)\min\{{v^C_1}'(\hat s_t),{v^C_2}'(\hat s_t) \} + \lambda V_{t+1}^{C,\lambda} & \text{if $t < H$,}\\
    \min\{{v^C_1}'(\hat s_H),{v^C_2}'(\hat s_H) \} &\text{if $t = H$}
\end{cases} \label{eq:safetycritictargets}
\end{equation}
Both safety critics $v^C_1$ and $v^C_2$ are regressed towards the TD-$\lambda$ targets with the following loss function,
\begin{equation}
    \mathcal{L}(\xi^C, v^C) = \mathbb{E}_{\pi^{\text{task}}, p_{\theta}} \left[\sum^{H-1}_{t=1} \frac{1}{2} (v^C(\hat s_t) - sg(V_t^{C,\lambda}))^2\right] \label{eq:safetycriticloss}
\end{equation}
where $(\xi^C, v^C) \in \{ (\xi^C_1, v^C_1), (\xi^C_2, v^C_2)\}$.
\subsection{Algorithm}
\label{sec:algorithm}
The full learning algorithm is split into two distinct phases: world model learning and policy optimisation (including safe policy synthesis and safety critic learning). To generate experience for world model learning we need to interact with the real environment and to mitigate safety violations in the real environment we pick actions with the shielded policy,
\begin{equation}
    \pi^{\text{shield}}(\cdot \: | \: s) = \begin{cases}
        \pi^{\text{task}}(\cdot \: | \: s) & \text{if $\hat \mu_{s\models\phi} \in [1 - \varepsilon + \epsilon, 1]$} \\
        \pi^{\text{safe}}(\cdot \: | \: s) & \text{otherwise} \label{eq:tracecosts}
    \end{cases} 
\end{equation}
To estimate $\mu_{s\models\phi}$ we roll-out the world model $p_{\theta}$ with the task policy $\pi^{\text{task}}$ to generate a batch of $m$ sequences of compact latent states $\langle\hat s^{(i)}_{1:H}\rangle^m_{i=1}$. For each trace $\tau^{(i)} = \hat s^{(i)}_1, ..., \hat s^{(i)}_H$ we compute the discounted cost as follows,
\begin{equation}
    \textrm{cost}(\tau^{(i)}) = \sum^H_{t=1} (\hat \gamma^{(i)}_t)^{t-1} \cdot \hat c^{(i)}_t 
\end{equation}
\begin{proposition}
    Under the `true' transition system $\mathcal{T}$ if $\textrm{cost}(\tau) < \gamma^{H-1} \cdot C$ then necessarily $\tau \models \square^{\leq H} \Phi$
\end{proposition}
\begin{proof}
The proof is a straightforward argument. By construction $c_t = C$ if and only if $\tau[t] \not \models \Phi$, therefore $\textrm{cost}(\tau) < \gamma^{H-1} \cdot C$ implies that $\forall \: 1 \leq t \leq H \; c_t = 0$ which implies that $\forall \: 1 \leq t \leq H \; \tau[t] \models \Phi$. 
\end{proof}
Using this idea, our estimate $\hat \mu_{s\models\phi} \approx \mu_{s\models\phi}$ is then computed as follows,
\begin{equation}
    \hat \mu_{s\models\phi} = \frac{1}{m} \sum^{m}_{i=1} \mathbbm{1}\left( \textrm{cost}(\tau^{(i)}) < \gamma^{H-1} \cdot C \right) \label{eq:muestimate}
\end{equation}
If we train safety critics then we can use the bootstrapped costs instead,
\begin{equation}
    \textrm{b-cost}(\tau^{(i)}) = \left(\sum^{H-1}_{t=1} (\hat \gamma^{(i)}_t)^{t-1} \cdot \hat c^{(i)}_t \right) + \min\left\{v_1^C(\hat s^{(i)}_H), v_2^C(\hat s^{(i)}_H) \right\} \label{eq:bootstrappedcosts}
\end{equation}
And we can estimate $\mu_{s\models\phi}$ with a larger horizon $T > H$, since the safety critics should capture the expected costs from $\hat s_{H}^{(i)}$ and beyond,   
\begin{equation}
    \hat \mu_{s\models\phi} = \frac{1}{m} \sum^{m}_{i=1} \mathbbm{1}\left( \textrm{b-cost}(\tau^{(i)}) < \gamma^{T-1} \cdot C \right) \label{eq:bootstrappedmu}
\end{equation}
After several environment interactions with the shielding policy $\pi^{\text{shield}}$, a batch of data $B$ is sampled from the replay buffer $\mathcal{D}$, for world model learning, task policy optimisation, safe policy optimisation and safety critic learning. The full algorithm is presented on the following page. 
\begin{algorithm}[!h]
\caption{DreamerV2 \cite{hafner2020mastering} with Shielding}
\raggedright
\label{alg:dreamerV2}
\textbf{Initialise:} replay buffer $\mathcal D$ with $S$ random epsiodes.\\
\textbf{Initialise:} $\theta$, $\psi^{\text{task}}$, $\psi^{\text{safe}}$, $\xi^{\text{task}}$, $\xi^{\text{safe}}$, $\xi^C_1$, $\xi^C_2$, ${\xi^C_1}'$, ${\xi^C_2}'$ randomly.\\
\begin{algorithmic}
\While{not converged}
\State \textit{// World model learning}
\State Sample $B \sim \mathcal{D}$.
\State For every $o_t \in B$ compute sequences $\hat s_{t:t+H}$ with RSSM.
\State Update RSSM parameters $\theta$ with Eq. \ref{eq:rssmloss} 
\State \textit{// Task policy optimisation}
\State From every $o_t \in B$ imagine sequences $\hat s_{t:t+H}$ with $\pi^{\text{task}}$.
\State Compute TD-$\lambda$ targets with Eq. \ref{eq:tasktargets}. 
\State Update task critic parameters $\xi^{\text{task}}$ with Eq. \ref{eq:taskcriticloss}.
\State Update task policy parameters $\psi^{\text{task}}$ with Eq. \ref{eq:taskpolicyloss}.
\State \textit{// Safety critic optimisation}
\State Compute safety critic targets with Eq. \ref{eq:safetycritictargets}.
\State Update safety critic parameters $\xi^C_1$ and $\xi^C_1$ with Eq. \ref{eq:safetycriticloss}.
\State For $i \in [1, 2]$ ${\xi^C_i}' \gets \nu \xi^C_i + (1 - \nu) {\xi^C_i}'$ (soft update \cite{fujimoto2018addressing}).
\State \textit{// Safe policy optimisation}
\State From every $o_t \in B$ imagine sequences $\hat s_{t:t+H}$ with $\pi^{\text{safe}}$.
\State Compute TD-$\lambda$ targets with Eq. \ref{eq:safecritictargets}. 
\State Update safe critic parameters $\xi^{\text{safe}}$ with Eq. \ref{eq:safecriticloss}.
\State Update safe policy parameters $\psi^{\text{safe}}$ with Eq. \ref{eq:safepolicyloss}.
\State \textit{// Environment interaction}
\For{$k = 1, ..., K$} 
    \State Observe $o_t$ from environment and compute $\hat s_t = (z_t, h_t)$.
    \State From $\hat s_t$ sample $m$ sequences $\langle\hat s^{(i)}_{1:H}\rangle^m_{i=1}$ with $\pi^{\text{task}}$.
    \State Estim. $\hat \mu_{s\models\phi} \approx \mu_{s\models\phi}$ with safety critics, Eq. \ref{eq:bootstrappedcosts} and Eq. \ref{eq:bootstrappedmu}.
    \State Play $a \sim \pi^{\text{shield}}(a \: | \: \hat s_t)$ and observe $r_t, o_{t+1}$ and $L(s_t)$.
    \State Construct $c_t$ with Eq. \ref{eq:costtarget} and $\gamma^{\text{safe}}_t$ with Eq. \ref{eq:safetydiscounttarget}.
    \State Append $\langle o_t, a_t, r_t, c_t, \gamma^{\text{safe}}_t, o_{t+1} \rangle$ to $\mathcal{D}$.
\EndFor
\EndWhile
\end{algorithmic}
\end{algorithm}

%%%%%%%%%%%%%%%%%%%%%%%%%%%%%%%%%%%%%%%%%%%%%%%%%%%%%%%%%%%%%%%%%%%%%%%%

\section{Experiments}

In this section we conduct a simple analysis and compare our algorithm, DreamerV2 with shielding, to DreamerV2 without shielding. We present results for two Atari games with state dependent labels: Assault and Seaquest (see Fig. \ref{fig:atari}). We start by giving a summary of the environments, followed by the experimental results and an accompanying discussion.

\subsection{Assault}

\emph{Assault} is a fixed shooter game similar to \emph{Space Invaders}. The goal is to shoot and destroy alien ships continuously deployed by a mothership. The smaller ships shoot lasers at the player which the player must avoid, otherwise they loose a life. In addition, the player's weapon can overheat if they fire too often, which also results in them loosing a life. The state dependent formula $\Phi$ that the agent aims to satisfy at each timestep is given as follows,
\begin{equation}
	\Phi = \neg \textbf{hit} \land \neg \textbf{overheat}
\end{equation}
where $\textbf{hit} = \textrm{true}$ iff the player is hit by a laser and $\textbf{overheat} = \textrm{true}$ iff the player's weapon overheats. We chose this environment because \citeauthor{giacobbe2021shielding} demonstrated state-of-the-art agents only concerned with reward overheat the weapon frequently and that BPS \cite{giacobbe2021shielding} alleviated the issue to some degree. The idea is that when the task policy $\pi^{\text{task}}$ is about to overheat the weapon the shield kicks in and the safe policy $\pi^{\text{safe}}$ prevents the agent from firing the weapon while avoiding any incoming lasers. 

\subsection{Seaquest}

\emph{Seaquest} is an underwater shooter in which the player controls a submarine equipped with an infinite supply of missiles. The goal is to rescue divers, shoot enemy sharks and submarines, while managing a limited supply of oxygen and resurfacing when necessary. The player receives points and a full supply of oxygen when they surface with a diver on board and if they surface with six divers they are awarded additional points based on the amount of oxygen they have left. However, surfacing without any divers is not permitted and results in the player loosing a life. The state formula $\Phi$ for Seaquest is a little more involved and is defined as follows,
\begin{multline}
    \Phi = (\textbf{surface} \Rightarrow ((\textbf{diver} \land \textbf{low-oxygen}) \lor \textbf{very-low-oxygen} \lor \\ \textbf{six-divers})) \land \neg \textbf{out-of-oxygen} \land \neg \textbf{hit}
\end{multline}
where $\textbf{surface} = \textrm{true}$ iff the submarine surfaces, $\textbf{diver} = \textrm{true}$ iff the submarine has at least one diver on board, $\textbf{low-oxygen} = \textrm{true}$ iff the players oxygen supply < 16, $\textbf{very-low-oxygen} = \textrm{true}$ iff the players oxygen supply < 4, $\textbf{six- divers} = \textrm{true}$ iff the submarine has six divers on board, $\textbf{out-of-oxygen} = \textrm{true}$ iff the player runs out of oxygen and  $\textbf{hit}$ is defined similarly as before. 

In words, it is only permissible to surface if the agent has a diver and is low on oxygen, has six divers or has very low on oxygen (with or without a diver). Surfacing with a diver when oxygen supplies are plentiful is deemed unsafe since it makes the game unnecessarily harder.

With Seaquest the agent needs to balance multiple objectives at once which is why it is a useful environment to test our approach. In our experiments we demonstrate that the learned safe policy $\pi^{\text{safe}}$ is able to deal with a slightly more complex set of constraints and prevent the agent from making costly mistakes during training.

\begin{figure}[t!]
     \centering
     \begin{subfigure}[b]{0.3\textwidth}
         \centering
         \includegraphics[width=\textwidth]{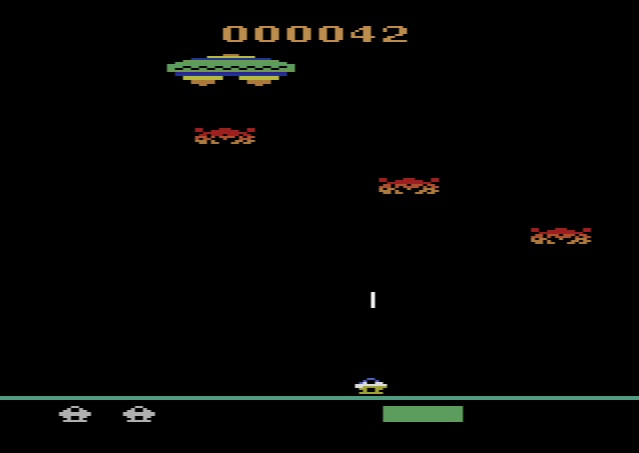}
         \caption{Assault}
         \label{fig:assault}
     \end{subfigure}
     \hfill
     \begin{subfigure}[b]{0.3\textwidth}
         \centering
         \includegraphics[width=\textwidth]{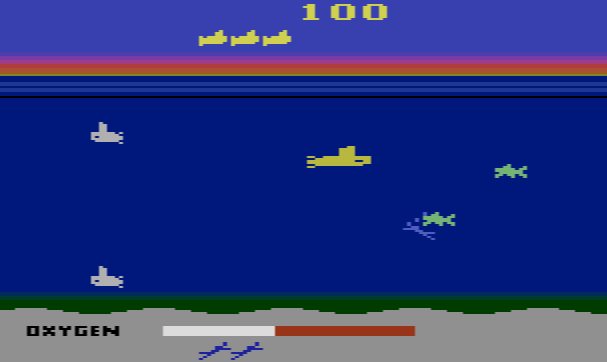}
         \caption{Seaquest}
         \label{fig:seaquest}
     \end{subfigure}
        \caption{Screenshots from the two Atari environments.}
        \label{fig:atari}
        \Description{Screenshots from Atari Assault and Atari Seaquest.}
\end{figure}

\subsection{Training Details}

The agents are trained on a single Nvidia Tesla A30 (24GB RAM) GPU and a 24-core/48 thread Intel Xeon CPU with 256GB RAM. Due to time constraints and limited compute resources all agents are trained on one seed and for precisely 40M frames on Atari environments provided by the ALE \cite{bellemare13arcade, machado18arcade}.

All the hyperparameters for DreamerV2 are set as their default values for Atari games, which are given in \cite{hafner2020mastering}. Notably, for all experiments we set the imagination horizon $H=15$, TD-$\lambda$ discount $\lambda=0.95$ and discount factor $\gamma=0.999$.

The cost and safety-discount predictors are implemented as neural networks with identical architectures to the reward and discount predictors used in DreamerV2. The safe policy, critic and safety critics are also implemented as neural networks in the same way that the task policy and critic are implemented in DreamerV2. See \cite{hafner2020mastering} for all details. The shielding hyperparameters are also fixed in all experiments as follows, specifically we set the bounded safety parameter $\varepsilon = 0.1$, number of samples $m=512$, approximation error $\epsilon = 0.09$\footnote{The gives us roughly $\delta=0.1$, using a tighter bound than Eq. \ref{eq:boundonm} which bounds the probability of overestimating $\mu_{s \models \phi}$.}, shield horizon $T = 30$ (2 seconds in real time) and safety critic smooth parameter updates $\nu = 0.005$.

\subsection{Results}
We evaluate our algorithm by comparing the performance of DreamerV2 \cite{hafner2020mastering} with and without shielding. Specifically, we compare the reward curves during training, the best episode return and the cumulative violations during training. Table \ref{tab:results} presents the best episode scores and total violations during training for DreamerV2 and DreamerV2 with shielding. In addition, Fig. \ref{fig:results}  displays the learning curves for both algorithms.

\begin{figure}
     \centering
     \begin{subfigure}[b]{0.32\textwidth}
         \centering
         \includegraphics[width=\textwidth]{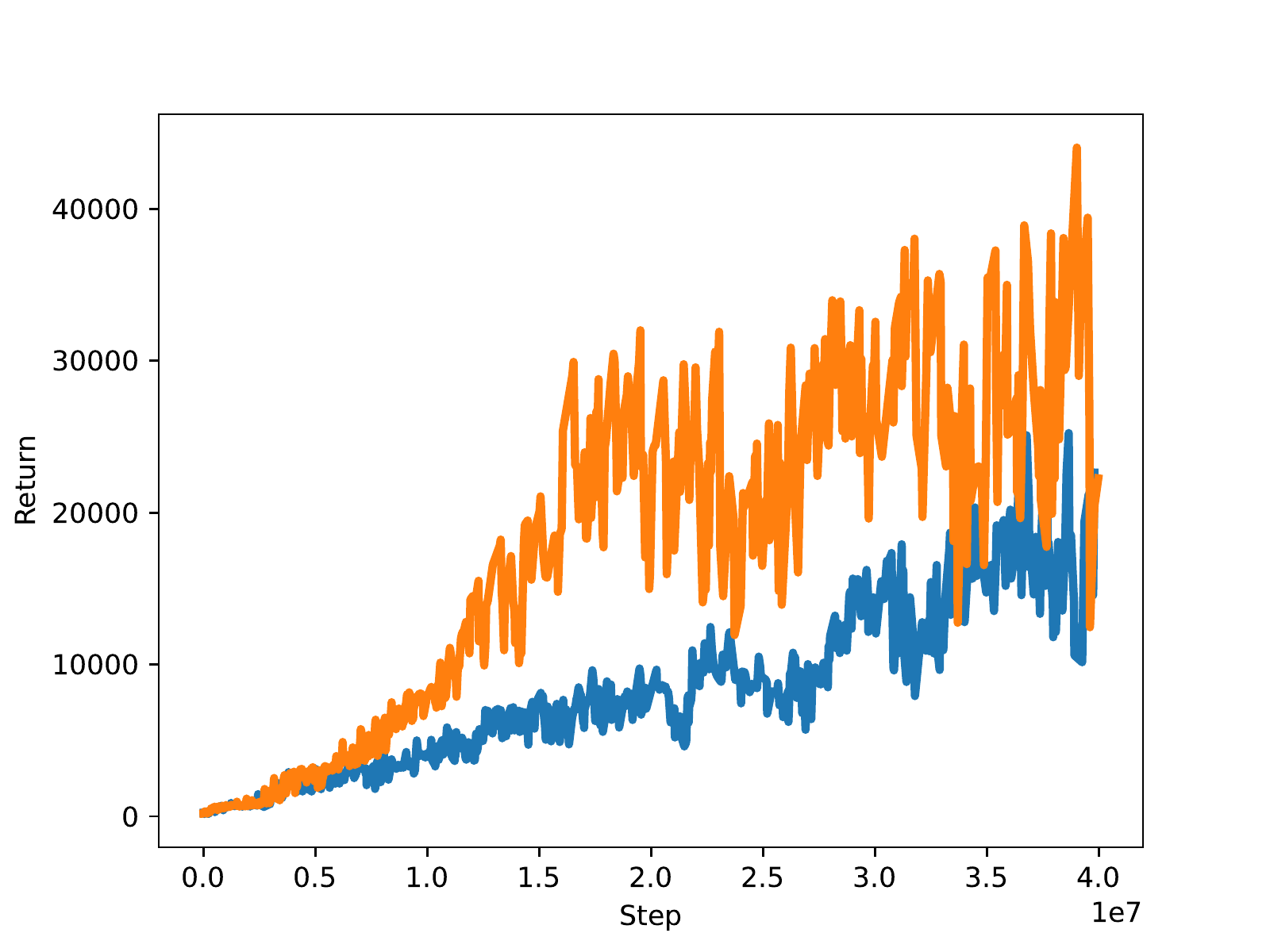}
         \caption{Training reward curve for Assault}
         \label{fig:assaultreturn}
     \end{subfigure}
     \hfill
     \begin{subfigure}[b]{0.32\textwidth}
         \centering
         \includegraphics[width=\textwidth]{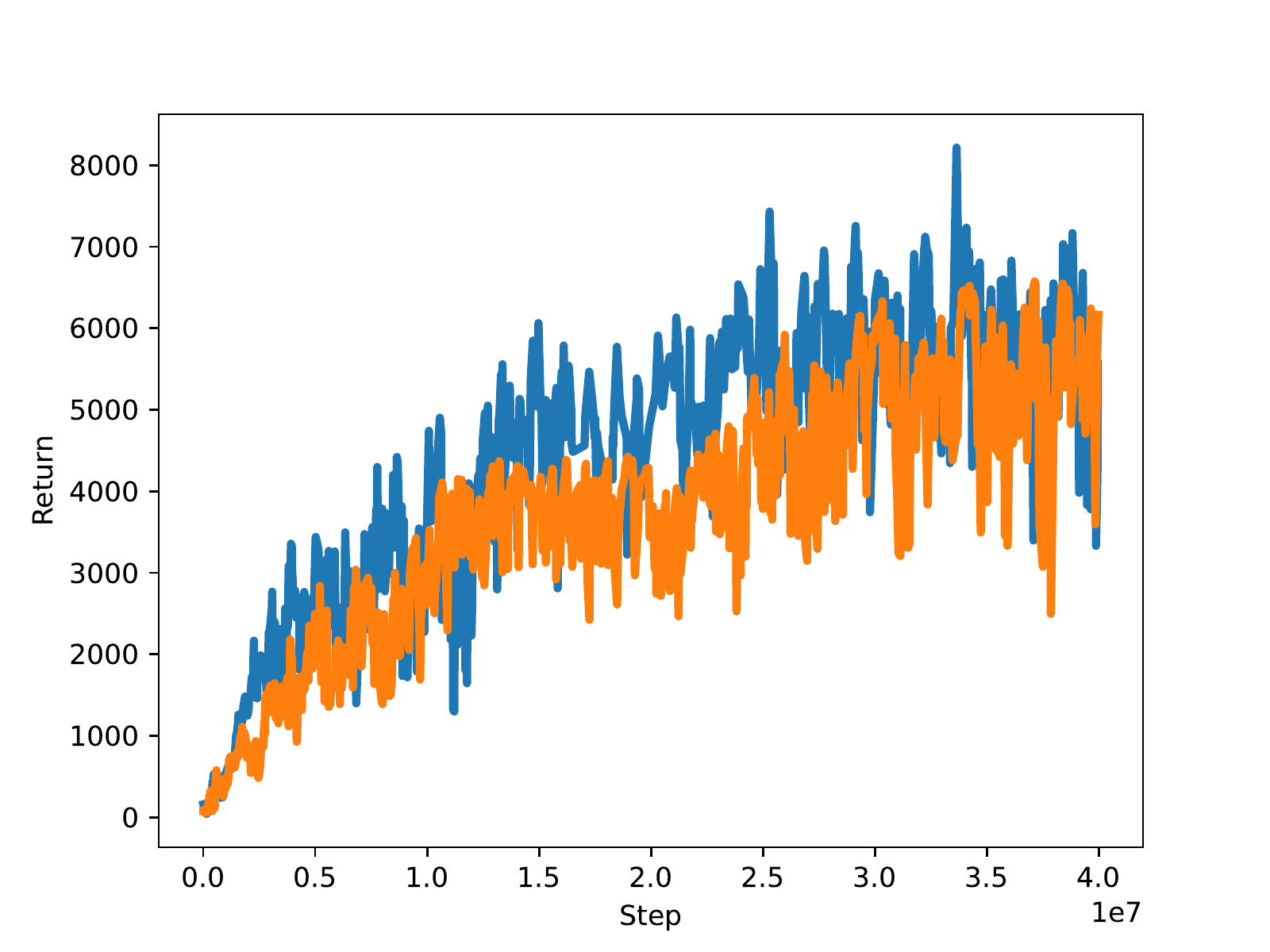}
         \caption{Training reward curve for Seaquest}
         \label{fig:seaquestreturn}
     \end{subfigure}
     \hfill
     \begin{subfigure}[b]{0.32\textwidth}
         \centering
         \includegraphics[width=\textwidth]{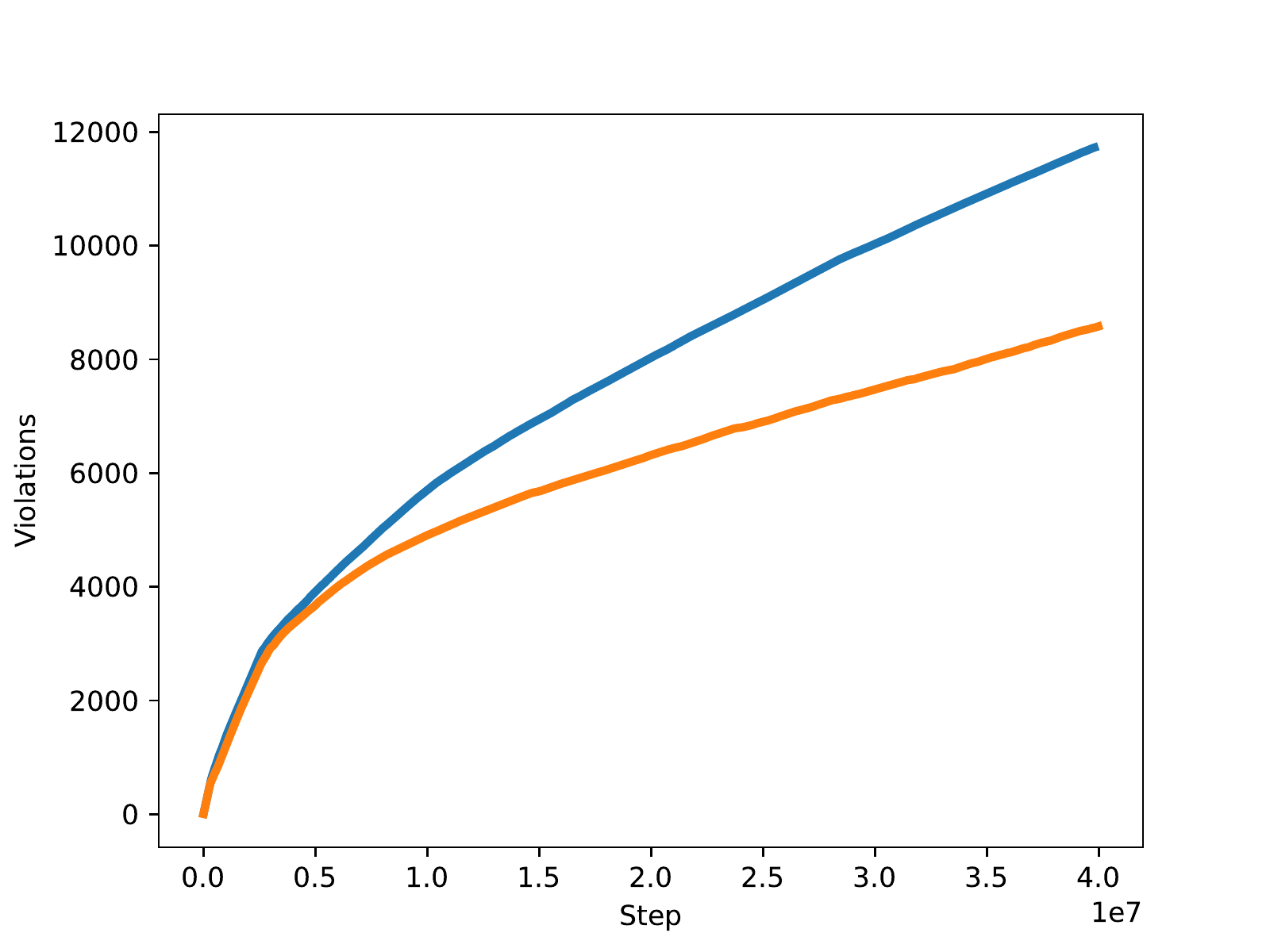}
         \caption{Cumulative violations for Assault}
         \label{fig:assaultviolations}
     \end{subfigure}
     \hfill
     \begin{subfigure}[b]{0.32\textwidth}
         \centering
         \includegraphics[width=\textwidth]{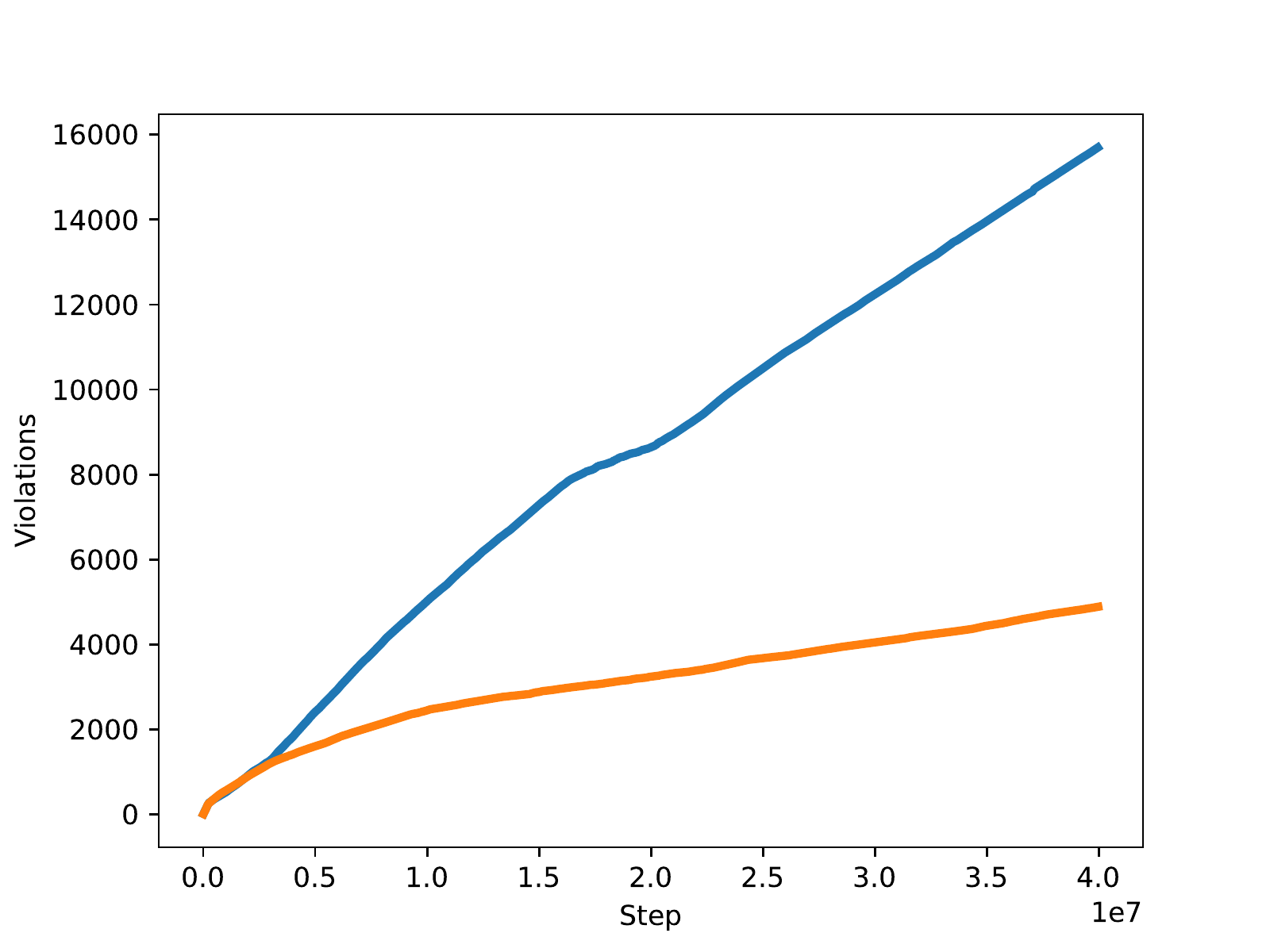}
         \caption{Cumulative violations for Seaquest}
         \label{fig:seaquestviolations}
     \end{subfigure}
     \hfill
     \includegraphics[width=0.35\textwidth]{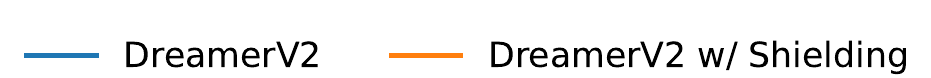}
        \caption{Training curves for DreamerV2 \cite{hafner2020mastering} and DreamerV2 with approximate shielding (ours).\textsuperscript{*}}
        \raggedright
        \small\textsuperscript{*} The reward curves are smoothed with simple exponential smoothing with $w=0.6$.
        \label{fig:results}
        \Description{Training reward curves and violations during training plots for Atari Assault and Atari Seaquest. In Atari Assault DeamerV2 with shielding obtains significantly better reward during training and fewer violations than DreamerV2 without shielding. In Atari Seaquest DeamerV2 with shielding makes fewer violations than DreamerV2 without shielding, although it obtains slightly worse performance.}
\end{figure}
{\renewcommand{\arraystretch}{1.4}
\setlength{\tabcolsep}{0.35em}
\begin{table}[h!]
\centering
\caption{Best episode scores and cumulative violations for for DreamerV2 \cite{hafner2020mastering} and DreamerV2 with approximate shielding.}
\begin{tabular}{ccccc}
     \toprule
     \multirow{2}{*}{Env} & \multicolumn{2}{c}{\emph{DreamerV2}} & \multicolumn{2}{c}{\emph{DreamerV2 w/ Shielding}}  \\
     & Best Score & \# Violations &  Best Score  & \# Violations \\
     \midrule
     Assault &34753&11726&\textbf{57504}&\textbf{8579}\\
     Seaquest &\textbf{11400}&15697&7040&\textbf{4889}\\
     \bottomrule
\end{tabular}
\label{tab:results}
\end{table}}

\paragraph{Discussion}

As seen in Table \ref{tab:results} and Fig. \ref{fig:results} our approximate shielding algorithm reduces the rate of safety violations for both \emph{Assault} and \emph{Seaquest}. In terms of reward, our shielding procedure has dramatically improved the speed of convergence for \emph{Assault} and maintained comparable performance for \emph{Seaquest}. We must note that these results are far from complete as we have not these run experiments over multiple random seeds or for the typical 200M frames, which is used as a common benchmark \cite{hafner2020mastering}. Nevertheless, we claim that our results provide compelling evidence that something is going on, which should motivate further investigation.

%%%%%%%%%%%%%%%%%%%%%%%%%%%%%%%%%%%%%%%%%%%%%%%%%%%%%%%%%%%%%%%%%%%%%%%%

\section{Related Work}

In this section we provide a discussion on the three main areas of research that our contribution is based on: world models, safe RL, and shielding.

\textbf{World Models}
%\emph{World models} 
were first introduced by \citeauthor{ha2018world} \cite{ha2018world} in a paper of the same name, although their inspiration is much more deeply rooted in psychology \cite{tolman1948cognitive} and Bayesian theories of the brain \cite{friston2003learning}. Dyna -- ``an integrated architecture for learning, planning and reacting'' -- proposed in \cite{sutton1991dyna}, introduced the idea of not only utilising reward signals to learn good policies, but also %introduced the idea 
learning a dynamics model through observed transitions \cite{sutton1991dyna}. In theory, planning with the learned dynamics model could speed up convergence of the policy, but many early approaches suffered from model bias \cite{atkeson1997comparison}. \emph{Gaussian processes} (GPs) \cite{williams2006gaussian} were quickly used as the stand-in dynamics model for the Dyna architecture as they reduced model bias by quantifying their own uncertainty \cite{deisenroth2011pilco}. However, GPs struggle in high-dimensional settings and so the use neural architectures has been increasingly explored instead.

More recently, with the neural architecture Dreamer \cite{hafner2019dream}, \citeauthor{hafner2019dream} demonstrated that policies can be learnt purely from imagined experience and transfer well to the original environment. Additionally, DreamerV2 \cite{hafner2020mastering} and DreamerV3 \cite{hafner2023mastering} demonstrated state-of-the-art performance in a variety of domains including the Atari benchmark \cite{bellemare13arcade, machado18arcade} and MineRL \cite{guss2019neurips} both of which have been notorious challenges for MBRL. 

Once a world model is learned it can be used in a flexible manner for policy optimisation \cite{hafner2019dream}, online planning schemes \cite{hafner2019learning, hung2022reaching, wu2022plan}, risk measures \cite{zhang2020cautious}, and defining intrinsic rewards for improved exploration \cite{sekar2020planning, latyshev2023intrinsic}. As a result world model have been applied in a variety of domains, such as robotics \cite{wu2022daydreamer}, imitation learning \cite{demossditto}, continual learning \cite{kessler2022surprising} and safe RL \cite{as2022constrained}.

\textbf{Safe RL}
\balance
%Safe RL 
is typically categorised as the problem of maximising reward, while maintaining some reasonable system performance during learning and deployment of the agent \cite{garcia2015comprehensive}. This definition has been interpreted in many different ways, stemming from different objectives in different domains. For example, reward hacking \cite{amodei2016concrete, skalse2022defining} refers to an agent `gaming' or exploiting a misspecified reward function, which can lead to undesired outcomes. Robustifying policies to distributional shift \cite{andrychowicz2020learning, peng2018sim, urpi2021risk} and the alignment problem \cite{russellnorvig2021, di2022goal} are also important areas of research in safe RL. However, we tackle the problem of \emph{safe exploration} \cite{pecka2014safe, garcia2015comprehensive} which can be described as the problem of minimising the violation of safety-constraints during the exploratory phase of training and beyond. 

The constrained Markov decision process (CMDP) \cite{altman1999constrained} is a widely used framework for modeling decision-making problems with safety constraints. In addition to maximising expected reward, agents must satisfy a set of constraints encoded as a cost function that penalises unsafe state-action pairs. In the tabular case, linear programs can be used to solve CMDPs \cite{altman1999constrained}. In the non-tabular case, a variety of model-free algorithms with function approximation have been proposed \cite{chow2017risk, achiam2017constrained, bohez2019value, liu2020ipo, yang2020projection}.

Model-based approaches for safe RL utilise a variety of different techniques for dynamics modelling and policy optimisation. \citeauthor{berkenkamp2017safe} use GPs to quantify model uncertainty in a principled way to safely learn neural network policies. Other approaches use ensembles of neural networks (NNs) to quantify uncertainty and either deploy MPC \cite{liu2020constrained, thomas2021safe}, perform policy optimisation within a certified region of the state space \cite{luo2021learning}, or use constrained policy optimisation with Lagrangian relaxation \cite{zanger2021safe} to learn safety-aware policies. Notable work by \citeauthor{as2022constrained} leverages Dreamer \cite{hafner2020mastering} and stochastic weight averaging Gaussian (SWAG) \cite{maddox2019simple} to obtain a Bayesian predictive distribution over possible world models that explain the dynamics of the environment. \citeauthor{as2022constrained} also stress the importance of policy optimisation with safety critics over shortsighted MPC schemes.

\textbf{Shielding for RL}
%Shielding for RL 
has been introduced as a correct by construction reactive ({\em shield}), which prevents the learned policy from entering unsafe states defined by some temporal logic formula \cite{alshiekh2018safe}. The shield itself can be applied before the agent picks an action (preemptive), modifying the action space of the agent. Alternatively, the shield can be applied after the agent picks an action (post-posed), overriding actions proposed by the agent if they lead to a violation. Both types of shield require the ability to construct and solve a safety game \cite{bloem2015shield} on a relatively compact representation of the MDP. Similar to control barrier functions (CBFs) \cite{ames2016control} from optimal control theory \cite{kirk2004optimal}, the shield projects the learned policy back into a verified safe set on the state space.

Recent work on shielding includes generalising it to partially observed \cite{carr2022safe} and multi-agent \cite{elsayed2021safe} settings, as well as resource constrained partially observed MDPs \cite{ajdarow2022shielding}. Many of these methods still require a suitable abstraction of the environment or sufficient domain knowledge for synthesising a shield. However, these strong assumptions come hand in hand with strong guarantees on safety of the learned policy, specifically \cite{alshiekh2018safe} show that by construction their 
%original 
shield synthesis procedure guarantees safety with minimal interference. 

Learning a shield online is an alternative approach to shielding RL policies without requiring significant prior knowledge. For example, \citeauthor{shperberg2022learning} propose tabular and parametric shields which are learning online to prevent agents from repeating catastrophic mistakes in the partially observed setting \cite{shperberg2022learning}. Other online shielding approaches include \emph{latent shielding} \cite{he2021androids} and BPS, both of which have substantially influenced our work. For a more complete review of reactive methods based on shielding we refer the interested reader to \cite{odriozola2023shielded}.
%%%%%%%%%%%%%%%%%%%%%%%%%%%%%%%%%%%%%%%%%%%%%%%%%%%%%%%%%%%%%%%%%%%%%%%%

\section{Conclusions}
In this paper we presented an approximate shielding algorithm for safe exploration of Atari agents and more general RL policies. Building on DreamerV2 \cite{hafner2020mastering} and previous work, such as, \emph{latent shielding} \cite{he2021androids} and BPS \cite{giacobbe2021shielding}, we propose a more general algorithm that uses safety critics and policy roll-outs to perform look-ahead shielding in the latent space of a learned world model. 

In contrast to previous work, we are able to successfully apply our approximate shielding algorithm with minimal hyperparameter tuning and no shielding introduction schedules. While we loose the benefit of strict and formal guarantees obtained by earlier shielding approaches \cite{alshiekh2018safe}, we are able to derive some probabilistic guarantees, although this is incomplete and further work should be done to derive bounds for the approximate transition system.

Nevertheless, our empirical results are promising and provide some good evidence that general RL agents can benefit from shielding in certain settings, not only in terms of complying with safety specifications, but also in terms of improved performance. The aim of this research is to shed light on the promise of this approach and we hope this opens the door to further investigation.
%%% The following command should be issued somewhere in the first column 
%%% of the final page of your paper.

%%%%%%%%%%%%%%%%%%%%%%%%%%%%%%%%%%%%%%%%%%%%%%%%%%%%%%%%%%%%%%%%%%%%%%%%

%%% The acknowledgments section is defined using the "acks" environment
%%% (rather than an unnumbered section). The use of this environment 
%%% ensures the proper identification of the section in the article 
%%% metadata as well as the consistent spelling of the heading.

\begin{acks}
This work was supported by UK Research and Innovation [grant number EP/S023356/1], in the UKRI Centre for Doctoral Training in Safe and Trusted Artificial Intelligence (\url{www.safeandtrustedai.org}).
\end{acks}

%%%%%%%%%%%%%%%%%%%%%%%%%%%%%%%%%%%%%%%%%%%%%%%%%%%%%%%%%%%%%%%%%%%%%%%%

%%% The next two lines define, first, the bibliography style to be 
%%% applied, and, second, the bibliography file to be used.

\bibliographystyle{ACM-Reference-Format} 
\bibliography{sample}

%%%%%%%%%%%%%%%%%%%%%%%%%%%%%%%%%%%%%%%%%%%%%%%%%%%%%%%%%%%%%%%%%%%%%%%%

\end{document}